\documentclass{article}

\usepackage{amsmath,amssymb}
\newcommand{\BlackBox}{\rule{1.5ex}{1.5ex}}  
\newenvironment{proof}{\par\noindent{\bf Proof\ }}{\hfill\BlackBox\\[2mm]}

\newtheorem{theorem}{Theorem}
\newtheorem{lemma}[theorem]{Lemma}

\newtheorem{remark}[theorem]{Remark}

\newtheorem{definition}[theorem]{Definition}

\DeclareMathOperator{\E}{\mathbb{E}}
\DeclareMathOperator{\I}{\mathbb{I}}
\RequirePackage{algorithm}
\RequirePackage{algorithmic}

\usepackage{microtype}
\usepackage{graphicx}
\usepackage{subfigure}
\usepackage{booktabs} 

\usepackage{hyperref}

\title{Profitable Bandits}

\author{
  Mastane Achab$^1$ \hspace{2em} Stephan Cl\'emen\c{c}on$^1$ \hspace{2em} Aur\'elien Garivier$^2$\\
  $^1$ LTCI, T\'el\'ecom ParisTech\\
  $^2$ IMT, Universit\'e de Toulouse\\
}

\begin{document}

\date{}
\maketitle

\begin{abstract}
Originally motivated by default risk management applications, this paper investigates a novel problem, referred to as the \emph{profitable bandit problem} here. At each step, an agent chooses a subset of the $K\geq 1$ possible actions. For each action chosen, she then receives the sum of a random number of rewards. Her objective is to maximize her cumulated earnings.
We adapt and study three well-known strategies in this purpose, that were proved to be most efficient in other settings: \textsc{kl-UCB}, \textsc{Bayes-UCB} and \textsc{Thompson Sampling}. For each of them, we prove a finite time regret bound which, together with a lower bound we obtain as well, establishes asymptotic optimality.
Our goal is also to \emph{compare} these three strategies from a theoretical and empirical perspective both at the same time. We give simple, self-contained proofs that emphasize their similarities, as well as their differences. While both Bayesian strategies are automatically adapted to the geometry of information, the numerical experiments carried out show a slight advantage for \textsc{Thompson Sampling} in practice.
\end{abstract}


\section{Introduction}

\subsection{Motivation}
\label{subsec:motivation}

A general and well-known problem for lenders and investors is to choose which prospective clients they should grant loans to, so as to manage credit risk and maximize their profit.
A classical supervised learning approach, referred to as \textit{credit risk scoring} consists in ranking all the possible profiles of potential clients, viewed through a collection of socio-economic features $Z$ by means of a (real valued) scoring rule $s(Z)$: ideally, the higher the score $s(Z)$, the higher the default probability. A wide variety of learning algorithms have been proposed to build, from a historical database, a scoring function optimizing ranking performance measures such as the ROC curve or its summary, the AUC criterion, see \textit{e.g.} \cite{West00}, \cite{Thomas00}, \cite{Li04}, \cite{Yang07} or \cite{Freund04}: the \textit{credit risk screening} process then consists in selecting the prospects whose score is below a certain threshold. However, this approach has a serious drawback in general, insofar as new scoring rules are  often constructed from truncated information only, namely historical data (the input features $X$ and the observed debt payment behavior) corresponding to past clients, eligible prospects who have been selected by means of a previous scoring rule, jeopardizing thus the screening procedure when applied to prospects who would have been previously non eligible for credit.  Hence, the credit risk problem leads to an exploration vs exploitation dilemma there is no way around for: should clients be used for improving the credit risk estimates, or should they be treated according to the level of risk estimated when they arrive? Lenders thus need sequential strategies able to solve this dilemma.

For simplicity, here we consider the very stylized situation, where each category applies for a loan of the same amount in expectation, and is proposed the same interest rate. Extension of the general ideas developed in this paper to more realistic situations will be the subject of further research.
In this article, we propose a mathematical model that addresses this issue. We propose several strategies, prove their optimality (by giving a lower bound on the inefficiency of any strategy) and empirically compare their performance in numerical experiments.

\subsection{Model}
\label{subsec:model}
We assume that the population (of credit applicants) is stratified according to $K\geq 2$ categories $a\in \{1,\dots,K\}$.
For each category $a$, the credit risk is modelled by a probability distribution $\nu_a$. The amount of the loan for category $a$ is assumed to be equal in expectation to $\tau_a$.
We assume that at each step $t\in \{1, \dots, T\}$, where $T$ denotes the total number of time steps (or time horizon), the agent is presented $C_a(t)\ge 1$ clients of each category a. She must choose a subset $A_t\subset\{1,\dots,K\}$ of categories to which they grant the loans.
We denote by $X_{a,c,t}-L_{a, c, t}$ the profit brought by the client number $c$ of category $a$ at step $t$.
In addition, we assume that all loans $L_{a, c, t}$ for the same category $a$ have the same known expectation $\tau_a$.
We assume that the variables $\{X_{a,c,t}\}$ are independent, and that $X_{a,c,t}$ has distribution $\nu_a$ and expectation $\mu_a$.
We further assume that, for any category $a\in\{1, \dots, K\}$, all random variables in the collection $\big\{C_a(t), X_{a, c, t}\big\}_{1\le t\le T, 1\le c\le C_a(t)}$ are pairwise independent
and that the $C_a(t)$'s are lower bounded by $c_a^- \ge 1$ and upper bounded by $c_a^+ < +\infty$.
We denote by $\tilde{c}_a^+\in [c_a^-, c_a^+]$ an upper bound on $\{\E[C_a(1)],\; \ldots,\; \E[C_a(T)]\}$.

Here and throughout, a \emph{sequential strategy} is a set of mappings specifying for each $t$ which categories to choose at time $t$ given the past observations only.
In other words, denoting by $I_t = (X_{a,c,s}, C_a(s))_{1\leq s\leq t, a\in A_s, 1\leq c\leq C_a(s)}$ the vector of variables observed up to time $t\ge 1$, a strategy specifies a sequence $(A_t)_{t\ge 1}$ of random subsets such that, for each $t\ge 2$, $A_t$ is $\sigma(I_{t-1})$-measurable.

It is the goal pursued in this work to define a strategy maximizing the expected cumulated profit given by
\[S_T = \E\left[\sum_{t=1}^T \sum_{a=1}^K \mathbb{I}\{a\in A_t\} \sum_{c=1}^{C_a(t)} X_{a,c,t}-L_{a, c, t}\right]\]
This is equivalent to minimizing the \emph{expected regret}
\begin{equation*}
  \begin{split}
    R_T &= \sum_{a\in\mathcal{A}^*} \Delta_a \tilde{C}_a(T) - S_T\\
    &= \sum_{a\in\mathcal{A}^*} \Delta_a  \left(\tilde{C}_a(T) - \E[N_a(T)]\right)
    + \sum_{a\notin\mathcal{A}^*}|\Delta_a|\E[N_a(T)],
  \end{split}
\end{equation*}
where $\tilde{C}_a(T) = \E\left[\sum_{t=1}^T C_a(t) \right]$ is the expected total number of clients from category $a$ over the $T$ rounds,
$\E[N_a(t)] = \E\left[\sum_{s=1}^t C_a(s)\I\{a\in A_t\}\right]$ is the expected number of observations from category $a$ up to time $t\ge 1$,
$\Delta_a = \mu_a-\tau_a$ is the (unknown) expected profit provided by a client of category $a$
and $\mathcal{A}^* =\{a\in \{ 1, \dots, K \}, \Delta_a > 0\}$ is the set of profitable categories.

\subsection{Applicative example}

Let us consider the credit risk problem in which a bank wants to identify categories of the population they should accept to loan.
It may be naturally formulated as a bandit problem with $K$ arms representing the $K$ categories of the population considered.
The bank pays $\tau_a$ when loaning to any member of some category $a\in \{ 1, \dots, K \}$. Each client $c\in\{1, \dots, C_a(t)\}$, belonging to category $a$,
who received a loan from the bank at time step $t$ is characterized by her capacity to reimburse it, namely the Bernoulli r.v. $B_{a, c, t}\sim \mathcal{B}(p_a)$:
\begin{itemize}
  \item $B_{a, c, t}=0$ in case of credit default, occurs with probability $1-p_a$: the bank gets no refunding,
  \item $B_{a, c, t}=1$ otherwise, occurs with probability $p_a$: the bank gets refunded $(1+\rho_a)\tau_a$ with $\rho_a$ the interest rate.
\end{itemize}
All individuals from the same category are considered as independent i.e. the $B_{a, c, t}$'s are i.i.d. realizations of $\mathcal{B}(p_a)$.
Hence the refunding $X_{a, c, t}$ received by the bank writes as follows: $X_{a, c, t} = (1+\rho_a)\tau_a B_{a, c, t}$.
Hence the bank should accept to loan to people belonging to all categories $a\in\{1, \dots, K\}$ such that $\mathbb{E}[X_{a, 1, 1}]>\tau_a$.
This condition rewrites:
\begin{equation}
  \label{eq:condition_bernoulli}
  p_a > \frac{1}{1+\rho_a}.
\end{equation}
Hence the role of the bank is to sequentially identify categories verifying Eq. \eqref{eq:condition_bernoulli} in order to maximize its cumulative profits over the $T$ rounds.

\subsection{State of the art}

In the multi-armed bandit (MAB) problem, a learner has to sequentially explore and exploit different sources in order to maximize the cumulative gain.
In the stochastic setting, each source (or \emph{arm}) is associated with a distribution generating random rewards. The optimal strategy in hindsight then consists in always pulling the arm
with highest expectation. Many approaches have been proposed for solving this problem such as the \textsc{UCB1} algorithm (\cite{auer2002finite}) for bounded rewards,
the \textsc{Thompson Sampling} heuristic first proposed in \cite{thompson1933likelihood}. More recently many algorithms have been proven to be asymptotically optimal, particularly in the case of exponential family distributions,
such as \textsc{kl-UCB} (\cite{2011arXiv1102.2490G}), \textsc{Bayes-UCB} (\cite{2016arXiv160101190K}) and \textsc{Thompson Sampling} (\cite{kaufmann2012thompson}, \cite{korda2013thompson}).
In this paper we consider a variation of the MAB problem, where, at each time step, the learner may pull several arms simultaneously or no arm at all.
To each arm is associated a known threshold and the goal is to maximize the cumulative profit which sums, for each arm pulled by the learner, the difference between the mean reward and the corresponding threshold.
This threshold is typically the price to pay for observing a reward from a given arm, e.g. a coin that has to be inserted in a slot machine.
Here the optimal strategy consists in always pulling the arms whose expectations are above their respective thresholds.
A very similar problem has been tackled in \cite{2016arXiv160508671L} in a best arm identification setting with fixed time horizon and for a unique threshold.
Nevertheless we believe that our profit maximization framework can be more relevant in many applications (e.g. bank loan management, see Section \ref{subsec:motivation}), where the learner wants
to learn while making profit.

In this paper we mainly focus on deriving (almost) optimal algorithms in the case of one-dimensional exponential family distributions.
Section \ref{sec:lower_bound} contains an asymptotic lower bound for the profitable bandit problem for any \emph{uniformly efficient} policy.
The three following sections (respectively \ref{sec:ucb}, \ref{sec:bayes_ucb} and \ref{sec:thompson}) are devoted to the adaptation of three celebrated MAB strategies (respectively \textsc{kl-UCB}, \textsc{Bayes-UCB} and \textsc{Thompson Sampling}) to the present problem.
We provide in each case a finite-time regret analysis. Asymptotical optimality properties of these algorithms are discussed in Section \ref{sec:optimality}.
The final Section \ref{sec:experiments} contains an empirical comparison of the three strategies through numerical experiments.


\section{Lower Bound}
\label{sec:lower_bound}

The goal of this section is to give an asymptotic lower bound on the regret of any \emph{uniformly efficient} strategy.
In this purpose, we adapt the argument of \cite{lai1985asymptotically}, rewritten by~\cite{2016arXiv160207182G}, on asymptotic lower bounds for the regret in MAB problems.
First we define a model $\mathcal{D} = \mathcal{D}_1 \times \dots \times \mathcal{D}_K$ where, for a any arm $a\in\{1, \dots, K\}$, $\mathcal{D}_a$ is the set of candidates for distribution-threshold pairs $(\nu_a, \tau_a)$.
Then, we introduce the class of \emph{uniformy efficient} policies that we focus on.

\begin{definition}
  A strategy is uniformly efficient if for any profitable bandit problem $(\nu_a, \tau_a)_{1\le a\le K} \in \mathcal{D}$,
  it satisfies for all arms $a\in\{1, \dots, K\}$ and for all $\alpha \in ]0, 1]$, $\mathbb{E}[N_a(T)]=o(\tilde{C}_a(T)^\alpha)$ if $\mu_a < \tau_a$ or $\tilde{C}_a(T)-\mathbb{E}[N_a(T)]=o(\tilde{C}_a(T)^\alpha)$ if $\mu_a > \tau_a$.
\end{definition}

In other words, such policies are not significantly under-performing. Now we can state our lower bound which applies to these reasonable strategies.

\begin{theorem}
  \label{thm:lower_bound}
  For all models $\mathcal{D}$, for all uniformly efficient strategies, for all profitable bandit problems $(\nu_a, \tau_a)_{1\le a\le K} \in \mathcal{D}$, for all non-profitable arms $a$ such that $\mu_a < \tau_a$,
  \begin{equation*}
    \liminf_{T\to \infty} \frac{\mathbb{E}[N_a(T)]}{\log T} \ge \frac{1}{\mathcal{K}_{\inf}(\nu_a, \tau_a)},
  \end{equation*}
  with $\mathcal{K}_{\inf}(\nu_a, \tau_a)$ defined by
  \[
  \mathcal{K}_{\inf}(\nu_a, \tau_a) = \inf\{ \text{KL}(\nu_a, \nu'_a), (\nu'_a, \tau_a) \in \mathcal{D}_a, \mu'_a>\tau_a \},
  \]
  where $KL(\nu_a, \nu'_a)$ denotes the Kullback-Leibler divergence between distributions $\nu_a$ and $\nu'_a$ and $\mu'_a$ is the expectation of distribution $\nu'_a$.
\end{theorem}

In the remainder of the article, we mainly focus on proposing asymptotically optimal strategies inspired by classical algorithms for MAB, namely \textsc{kl-UCB} (\cite{2011arXiv1102.2490G} and \cite{2012arXiv1210.1136C}), \textsc{Bayes-UCB} (\cite{2016arXiv160101190K}) and \textsc{Thompson Sampling} (\cite{kaufmann2012thompson} and \cite{korda2013thompson}).
For each policy, we  prove a corresponding upper bound on its regret which will be hopefully tight with respect to the lower bound stated above.


\section{Preliminaries}

\subsection{One-dimensional Exponential Family}
We consider arms with distributions belonging to a one-dimensional exponential family. It should be noted that the \textsc{kl-UCB-4P} algorithm presented next, as \textsc{kl-UCB}, can be shown to apply to the non-parametric setting of bounded distributions, although the resulting approach has weaker optimality properties (see Section~\ref{subsec:bounded_rewards}).

\noindent {\bf Definition and properties.} A one-dimensional canonical exponential family is a set of probability distributions $\mathcal{P}=\{ \nu_\theta, \theta\in \Theta \}$ indexed by a real parameter $\theta$ living in the parameter space $\Theta=]\theta^-, \theta^+[\subseteq \mathbb{R}$
and where for all $\theta \in \Theta$, $\nu_\theta$ has a density $f_\theta(x) = A(x)\exp(G(x)\theta-F(\theta))$ with respect to a reference measure $\xi$.
$A(x)$ and the sufficient statistic $G(x)$ are fixed functions that characterize the exponential family and $F(\theta) = \log \int A(x) \exp(G(x)\theta)d\xi(x)$ is the normalization function.
For notational simplicity, we only consider families with $G(x)=x$, which includes many usual distributions (\textit{e.g.} Gaussian, Bernoulli, Gamma among others)
but not heavy-tailed distributions such as Pareto or Weibull. Nevertheless generalizing all the results proved in this paper to a general sufficient
statistic $G(x)$ is straightforward and boils down to considering empirical sufficient statistics $\hat{g}(n)=(1/n)\sum_{s=1}^n G(X_s)$ instead of empirical means.
We additionally assume that $F$ is twice differentiable with a continuous second derivative (classic assumption, see e.g. \cite{wasserman2013all})
which implies that $\mu: \theta\mapsto \mathbb{E}_{\theta}[X]$ is strictly increasing and thus one-to-one in $\theta$.

The Kullback-Leibler divergence between two distributions $\nu_\theta$ and $\nu_{\theta'}$ in the same exponential family admits the following closed form expression as a function of the natural parameters $\theta$ and $\theta'$:
\begin{equation*}
K(\theta, \theta') := KL(\nu_\theta, \nu_{\theta'}) = F(\theta') - [F(\theta) + F'(\theta)(\theta' - \theta)].
\end{equation*}

We also introduce the KL-divergence between two distributions $\nu_{\mu^{-1}(x)}$ and $\nu_{\mu^{-1}(y)}$:
\begin{align}
  \label{eq:kullback_exp}
  d(x, y) :&= K(\mu^{-1}(x), \mu^{-1}(y))\\
  &= \sup_\lambda \{ \lambda x-\log\mathbb{E}_{\mu^{-1}(y)}[\exp(\lambda X)] \},
\end{align}
where the last equality comes from the proof of Lemma 3 in \cite{korda2013thompson}.
This last expression of $d$ allows to build a confidence interval on $x$ based on a fixed number $s$ of i.i.d. samples from $\nu_{\mu^{-1}(x)}$
by applying the Cram\'er-Chernoff method (see e.g. \cite{boucheron2013concentration}).

We mainly investigate the profitable bandit problem in the parametric setting, where all distributions $\{ \nu_{\theta_a} \}_{1\le a\le K}$ belong to a known one-dimensional canonical
exponential family $\mathcal{P}$ as defined above. Examples of such distributions are provided in the Supplementary Material.

\subsection{Index Policies}

All bandit strategies considered in this paper are \emph{index policies}: they are fully characterized by an index $u_a(t)$ which is computed at each round $t\ge 1$ for each arm separately; only arms with an index larger than the threshold $\tau_a$ are chosen.  Index policies are formally described in Algorithm~\ref{alg:tauklUCB}.

  \begin{algorithm}
    \caption{Generic index policy}
    \label{alg:tauklUCB}
    \begin{algorithmic}[1]
      \REQUIRE time horizon $T$, thresholds $(\tau_a)_{a\in\{1, \dots, K\}}$
      \STATE Pull all arms: $A_1 = \{1, \dots, K\}$
      \FOR {$t = 1$ \textbf{to} $T-1$}
        \STATE Compute $u_a(t)$ for all arms $a\in \{1, \dots, K\}$
        \STATE Choose $A_{t+1} \gets \left\{a\in \{1, \dots, K\},
        u_a(t) \ge \tau_a  \right\}$
      \ENDFOR
    \end{algorithmic}
  \end{algorithm}


\section{The \textsc{kl-UCB-4P} Algorithm}
\label{sec:ucb}

We introduce the  \textsc{kl-UCB-4P} algorithm as a variant of the  \textsc{UCB1} \cite{auer2002finite} algorithm and more precisely of its improvement \textsc{kl-UCB} \cite{2011arXiv1102.2490G}.
It is defined by the index \begin{multline*}u_a(t) = \sup\Big\{ q>\hat{\mu}_a(t):N_a(t) d(\hat{\mu}_a(t), q)\le \log t+c\log\log t \Big\}\;,\end{multline*}
where $d$ is the divergence induced by the Kullback-Leibler divergence defined in Equation~\eqref{eq:kullback_exp}, and where $c$ is a positive constant typically smaller than 3.
Due to its special importance for bounded rewards, we name  \textsc{kl-Bernoulli-UCB-4P} the case $d=d_\text{Bern}: (x, y)\mapsto xlog(x/y) + (1-x)log((1-x)/(1-y))$ and  \textsc{kl-Gaussian-UCB-4P} the choice $d=d_\text{Gauss}: (x, y)\mapsto 2(x-y)^2$.
\subsection{Analysis for one-dimensional exponential family}
We show for the \textsc{kl-UCB-4P} algorithm a finite-time regret bound that proves its asymptotic optimality up to a multiplicative constant $\tilde{c}_a^+/c_a^-$ (see Section \ref{sec:optimality} for further discussion).
To this purpose, we upper-bound the expected number of times non-profitable arms are pulled and profitable ones are not. The analysis is sketched below, while detailed proofs are deferred to the Supplementary Material.

\begin{theorem}
  \label{th:upper_bound_N}
    The \textsc{kl-UCB-4P} algorithm satisfies the following properties:
    \ \\(i). For any non-profitable arm $a\in \{1, \dots, K\} \setminus \mathcal{A}^*$ and all $\epsilon > 0$,
    \begin{equation*}
      \begin{split}
      \mathbb{E}[N_a(T)]
      \le (1+\epsilon) \frac{\tilde{c}_a^+(\log T + c \log\log T)}{c_a^- d(\mu_a, \tau_a)}
      + \tilde{c}_a^+\left\{1 + H_1(\epsilon)/T^{\beta_1(\epsilon)}\right\},
    \end{split}
    \end{equation*}
    where $H_1(\epsilon)$ and $\beta_1(\epsilon)$ are positive functions of $\epsilon$ depending on $c_a^-, \mu_a$ and $\tau_a$.\\
    (ii). For any profitable arm $a\in \mathcal{A}^*$, if $T\ge \max(3, c_a^+)$ and $c\ge 3$, we have:
    \begin{equation*}
      \tilde{C}_a(T)-\mathbb{E}[N_a(T)] \le \tilde{c}_a^+ \{e(2c+3)\log\log T + c_a^+ + 3\}.
    \end{equation*}
\end{theorem}

\subsection{Sketch of proof}
The analysis goes as follows:

(i). For a \textbf{non-profitable arm} $a\in\{1, \dots, K\}\setminus\mathcal{A}^*$, we must upper bound $\E[N_a(T)]$.
At first, a sub-optimal arm is drawn because its confidence bonus is large. But after some $K_T\approx \kappa\log(T)$ draws (where $\kappa$ is the information constant given in the theorem), the index $u_a(t)$ can be large only when the empirical mean of the observations deviates from its expectation, which has small probability. Thus, we write
\[
\E[N_a(T)] \le \tilde{c}_a^+\{ K_T + \sum_{t\ge 1} \mathbb{P}(a\in A_t, N_a(t)>K_T )] \}\;.\]
One obtains that $K_T$ gives the main term in the regret.
The contribution of the remaining sum is negligible: we observe that
\begin{equation*}
  \begin{split}
    &(a\in A_{t+1}) = (u_a(t)\ge \tau_a)\\
    &\subset (\hat{\mu}_a(t)<\tau_a, d(\hat{\mu}_a(t), \tau_a)\le d(\hat{\mu}_a(t), u_a(t)))\\
    &\subset \left(\hat{\mu}_a(t)<\tau_a, d(\hat{\mu}_a(t), \tau_a)\lesssim \frac{\log(t)+c\log\log(t)}{N_a(t)}\right),\;.
  \end{split}
\end{equation*}
As a deviation from the mean, the last event proved to have small probability when $N_a(t) > K_T$. Summing over these probabilities produce a term negligible compared to $K_T$.

(ii). For a \textbf{profitable arm} $a\in\mathcal{A}^*$, we must upper bound $\tilde{C}_a(T) - \E[N_a(T)]$. We write
\[
\tilde{C}_a(T) - \E[N_a(T)] \le \tilde{c}_a^+ \sum_{t=1}^{T-1} \mathbb{P}(a \notin A_{t+1})\;,\]
and we control the defavorable events by noting that
\[ 
     (a \notin A_{t+1}) =(u_a(t)<\tau_a)\\
     \subset (u_a(t)<\mu_a)\;,\]
where the probability of the last event can be upper bounded by means of a
self-normalized deviation inequality such as in Lemma 10 in \cite{2012arXiv1210.1136C}.

\subsection{Extension to General Bounded Rewards}
\label{subsec:bounded_rewards}
In this subsection, rewards bounded in $[0, 1]$ are considered and we build confidence intervals $u_a(t)$ with Bernoulli and Gaussian KL divergence, i.e. $d = d_\text{Bern}$ or $d=d_\text{Gauss} \}$,
which respectively define \textsc{kl-Bernoulli-UCB-4P} and \textsc{kl-Gaussian-UCB-4P} algorithms.
Then ,with the same proof as in the one-dimensional exponential family setting, we obtain
similar guarantees as in Theorem \ref{th:upper_bound_N} except that the divergence $d$
is either $d_\text{Bern}$ or $d_\text{Gauss}$.
By Pinsker's inequality, $d_\text{Bern}>d_\text{Gauss}$, which implies that \textsc{kl-Bernoulli-UCB-4P} performs always better than \textsc{kl-Gaussian-UCB-4P}.
However, this upper bound is not tight w.r.t. the lower bound stated in Theorem \ref{thm:lower_bound} obtained for general bounded distributions.
Hence, none of these two approaches is asymptotically optimal.


\section{The \textsc{Bayes-UCB-4P} algorithm}
\label{sec:bayes_ucb}

\subsection{Analysis}

We now propose a Bayesian index policy which derived from \textsc{Bayes-UCB} (\cite{2016arXiv160101190K}).
For all arms $a\in \{ 1, \dots, K \}$, a prior is chosen on the mean $\mu_a$.
At each round $t\ge 1$, we compute the posterior distribution $\pi_{a, t} = \pi_{a, N_a(t), \bar{\mu}_{a}(t)}$ using the previous observations from arm $a$. We compute the quantile $\bar{q}_a(t) = Q(1-1/(t(\log t)^c); \pi_a(t))$,
where $Q(\alpha, \pi)$ denotes the quantile of order $\alpha$ of the distribution $\pi$.
The $\textsc{Bayes-UCB-4P}$ is the index policy defined by $u_a(t)=\bar{q}_a(t)$.
In other words, arm $a$ is pulled whenever the quantile of the posterior is larger than the threshold $\tau_a$.
The following results, proven in the Supplementary Material, show that \textsc{Bayes-UCB-4P} is asymptotically optimal up to a multiplicative constant $\tilde{c}_a^+/c_a^-$ (see Section \ref{sec:optimality}).

\begin{theorem}
  \label{th:upper_bound_N_bayesUCB}
  When running the \textsc{Bayes-UCB-4P} algorithm the following assertions hold.
  \ \\
  (i). For any non-profitable arm $a\in \{1, \dots, K\}\setminus \mathcal{A}^*$ and for all $\epsilon>0$ there exists a problem-dependent constant $N_a(\epsilon)$ such that for all $T\ge N_a(\epsilon)$,
  \begin{equation*}
    \begin{split}
    \mathbb{E}[N_a(T)]
    \le \left(\frac{1+\epsilon}{1-\epsilon} \right) \frac{\tilde{c}_a^+ (\log T + c \log\log T)}{c_a^- d(\mu_a, \tau_a)}
    + \tilde{c}_a^+ \left\{ 1 + H_2 + \frac{H_3(\epsilon)}{T^{\beta_2(\epsilon)}} \right\},
    \end{split}
  \end{equation*}
  where $H_2$, $H_3(\epsilon)$ and $\beta_2(\epsilon)$ are respectively a constant and two positive function of $\epsilon$ depending on $c_a^-, \mu_0^-, \mu_a$ and $\tau_a$.\\
  (ii). For any profitable arm $a\in \mathcal{A}^*$, if $T\ge t_a$ and $c\ge 5$,
  \begin{equation*}
    \begin{split}
    \tilde{C}_a(T)-\mathbb{E}[N_a(T)]
    \le \tilde{c}_a^+ \left\{\frac{e(2(c-2)+4)}{A}\log\log T + t_a + 1\right\},
  \end{split}
  \end{equation*}
  where $t_a=\max(e/A, 3, A, c_a^+, A c_a^+)$ and $A$ is a constant depending on the chosen prior distribution.
\end{theorem}

\subsection{Sketch of proof}
We present the main steps of the proof of Theorem \ref{th:upper_bound_N_bayesUCB}
(see the Supplementary Material for the complete version). The idea is to capitalise on the analysis of \textsc{kl-UCB-4P}, and to relate the quantiles of the posterior distributions to the Kullback-Leibler upper-confidence bounds.

(i). For a non-profitable arm $a\in\{1, \dots, K\}\setminus\mathcal{A}^*$, we want to upper bound $\E[N_a(T)]$. Again, we use the following decomposition:
\[
\E[N_a(T)] \le \tilde{c}_a^+\{ K_T + \sum_{t\ge 1} \mathbb{P}(a\in A_t, N_a(t)>K_T ) \},\]
where $K_T \approx \kappa\log(T)$ of the same order of magnitude as the asymptotic lower bound derived in Theorem \ref{thm:lower_bound}.
This cut-off $K_T$ is expected to be the dominant term in our upper bound, since the contribution of the remaining sum is negligible compared to $K_T$:  when $N_a(t) > K_T$, we first observe that
\begin{equation}
  \label{eq:control_bayes}
  \begin{split}
    (a\in A_{t+1}) = (\bar{q}_a(t)\ge \tau_a)
    = \Big(\pi_{a, t}([\tau_a, \mu^+[) \ge \frac{1}{t(\log t)^c}\Big),
  \end{split}
\end{equation}
where the $\pi_{a, t}$ is the posterior distribution on $\mu_a$ at round $t$ and $\bar{q}_a(t)$
is, under $\pi_{a, t}$, the quantile of order $1-\frac{1}{t(\log t)^c}$.
The key ingredient here is Lemma 4 from \cite{2016arXiv160101190K}, which relates a quantile of the posterior to an upper confidence bound on the empirical mean:
\begin{equation*}
  \pi_{a, t}([\tau_a, \mu^+[) \lesssim \sqrt{N_a(t)}e^{-N_a(t) d(\hat{\mu}_a(t), \tau_a)}.
\end{equation*}
This permits to conclude as for \textsc{kl-UCB-4P}.

(ii). For a profitable arm $a\in\mathcal{A}^*$, we must upper bound $\tilde{C}_a(T) - \E[N_a(T)]$. We write
\[
\tilde{C}_a(T) - \E[N_a(T)] \le \tilde{c}_a^+ \sum_{t=1}^{T-1} \mathbb{P}(a \notin A_{t+1}).\]
Then we note that for all $t\ge 1$,
\begin{equation*}
   \begin{split}
     (a \notin A_{t+1}) =(\bar{q}_a(t)<\tau_a)
     =(\pi_{a, t}([\tau_a, \mu^+[) < \frac{1}{t(\log t)^c}).
   \end{split}
\end{equation*}
Using again the bridge between posterior quantiles and upper-confidence bounds of Lemma 4 in \cite{2016arXiv160101190K}:
\begin{equation*}
  \pi_{a, t}([\tau_a, \mu^+[) \gtrsim \frac{e^{-N_a(t) d(\hat{\mu}_a(t), \tau_a)}}{N_a(t)}\;,
\end{equation*}
we can again argue as for \textsc{kl-UCB-4P}.



\section{The \textsc{TS-4P} Algorithm}
\label{sec:thompson}

\subsection{Analysis}

The \textsc{TS-4P} algorithm described in this section is inspired from the variant of
\textsc{Thompson Sampling} detailed in \cite{korda2013thompson}.
Although the analysis of \textsc{Bayes-UCB} in Section \ref{sec:bayes_ucb} is valid for any prior distribution
the Bayesian approach proposed in this section will be analyzed only for Jeffreys priors (see \cite{korda2013thompson} for more details).
Following the notations in \cite{korda2013thompson}, $\pi_{a, 0}$ will refer to the prior distribution on $\theta_a$
and $\pi_{a, t}$ to the posterior at the end of round $t$ (or, equivalently, at the beginning of round $t+1$).
At each round $t\ge 1$ the posterior distribution $\pi_a(t)$ on the parameter $\theta_a$ is updated and we sample $\theta_a(t)\sim \pi_a(t)$,
and define the \textsc{TS-4P} algorithm (see Algorithm \ref{alg:tauklUCB}) which pulls arm $a$ if $u_a(t)=\mu(\theta_{a, t})$ is larger or equal to $\tau_a$.

\begin{theorem}
  \label{th:upper_bound_N_thompson}
  When running the \textsc{TS-4P} algorithm the following assertions hold.
  \ \\(i). For any non-profitable arm $a\in \{1, \dots, K\} \setminus \mathcal{A}^*$ and for all $\epsilon \in ]0, 1[$,
  \begin{equation*}
    \mathbb{E}[N_a(T)] \le \left(\frac{1+\epsilon}{1-\epsilon} \right) \frac{\tilde{c}_a^+ \log T}{c_a^- d(\mu_a, \tau_a)} + H_4,
  \end{equation*}
  where $H_4$ is a problem dependent constant.\\
  (ii). For any profitable arm $a\in \mathcal{A}^*$,
  \begin{equation*}
    \tilde{C}_a(T)-\mathbb{E}[N_a(T)] \le H_5,
  \end{equation*}
  with $H_5$ a problem dependent constant.
\end{theorem}

\subsection{Sketch of proof}
Here we give the main steps of the proof of Theorem \ref{th:upper_bound_N_thompson}
(see the Supplementary Material for complete proof).

(i). For a non-profitable arm $a\in\{1, \dots, K\}\setminus\mathcal{A}^*$, we must upper bound $\E[N_a(T)]$.
We first write:
\[
\E[N_a(T)] \lesssim \tilde{c}_a^+\{ K_T + \sum_{t\ge 1} \mathbb{P}(a\in A_t, E_a(t), N_a(t)>K_T) \},\]
where $K_T \approx \log(T)$ is, as in the proofs of \textsc{kl-UCB-4P} and \textsc{Bayes-UCB-4P}, a cut-off corresponding to the main term in our bound as suggested by the asymptotic lower bound in Theorem \ref{thm:lower_bound}
and $E_a(t)$ is a high probability event ensuring that the current empirical mean at times $t$, namely $\hat{\mu}_a(t)$,
is well concentrated around the true mean $\mu_a$.
It remains to prove that the sum of defavorable events (for $N_a(t)>K_T$ and under $E_a(t)$) is negligible compared to $K_T$.
Observe that the following holds:
\begin{equation}
  \label{eq:control_thompson}
  \begin{split}
    \mathbb{P}(a\in A_t, E_a(t), N_a(t)>K_T)
    \le \mathbb{P}(\mu(\theta_a(t))\ge \tau_a, E_a(t), N_a(t)>K_T),
  \end{split}
\end{equation}
where $\theta_a(t)$ is sampled from the posterior distribution.
Then we upper bound the right-hand side expression in Eq. \ref{eq:control_thompson} thanks to the deviation inequality
stated in Theorem 4 in \cite{korda2013thompson} and that we recall in Lemma \ref{lem:posterior_concentration} in the Supplementary Material.
Summing over these probabilities produce a term negligible compared to $K_T$.

(ii). For a profitable arm $a\in\mathcal{A}^*$, we must upper bound $\tilde{C}_a(T) - \E[N_a(T)]$, whic we decompose as follows:
\[
\tilde{C}_a(T) - \E[N_a(T)] \le \tilde{c}_a^+ \sum_{t=1}^{T-1} \mathbb{P}(a \notin A_{t+1}).\]
Then, we control the defavorable events: for all $t\ge 1$,
\begin{equation*}
   \begin{split}
     \sum_{t=1}^{T-1} \mathbb{P}(a \notin A_{t+1})
     \lesssim \sum_{t=1}^{+\infty} \mathbb{P}(\mu(\theta_a(t)) < \tau_a, E_a(t) | N_a(t)>t^b) + \sum_{t=1}^{+\infty} \mathbb{P}(),
   \end{split}
\end{equation*}
where the first series is proved to converge thanks to Lemma \ref{lem:posterior_concentration}
and the second too by Lemma \ref{lem:N_larger_tb} provided in the Supplementary.
We point out that our proof of Lemma \ref{lem:N_larger_tb}, which is a much simplified version of the proof
of Proposition 5 in \cite{korda2013thompson}. This simplification relies on the fact that
the objective of profitable bandit presents between the different categories contrary to the classical
multi-armed bandit problem where all arms are compared and the goal is to find the best one.



\section{Asymptotic Optimality}
\label{sec:optimality}

A direct consequence of theorems \ref{th:upper_bound_N}, \ref{th:upper_bound_N_bayesUCB} and \ref{th:upper_bound_N_thompson} is the following asymptotic upper bound on the regret
of \textsc{kl-UCB-4P} (with $c\ge 3$), \text{Bayes-UCB-4P} (with $c\ge 5$) and \textsc{TS-4P}:
  \begin{equation*}
    \limsup_{T\to \infty} \frac{R_T}{\log T} \le \sum_{a,\, \mu_a<\tau_a} \frac{\tilde{c}_a^+ |\Delta_a|}{c_a^- d(\mu_a, \tau_a)}.
  \end{equation*}

Observe that this asymptotic upper bound on the regret is tight with the asymptotic lower bound in Section \ref{sec:lower_bound}
when $\tilde{c}_a^+ = c_a^-$ for all non-profitable arms $a\in\{1, \dots, K\}\setminus\mathcal{A}^*$, which is achieved when the corresponding $\{C_a(t)\}_{1\le t\le T}$ are constant.
In this particular case these three algorithms are asymptotically optimal.


\section{Numerical Experiments}
\label{sec:experiments}

We perform three series of numerical experiments for three different one-dimensional exponential families:
Bernoulli, Poisson and Exponential. In each scenario, we consider five arms ($K=5$) whose associated distributions belong to the same
one-dimensional exponential family and among which two are profitable ($|\mathcal{A}^*| = 2$).
We always choose the $C_a(t)$ such that $(C_a(t)-1)$ follows a Poisson distribution $\mathcal{P}(\lambda_a)$
where the respective values for the $\lambda_a$ of each category are $(3, 4, 5, 6, 7)$.
Moreover, the time horizon is chosen equal to $T=10000$ and the regret curves result from empirically averaging over $10000$ independent trajectories.
Our experiments also include algorithms, all index policies, whose theoretical properties have not been discussed in this article, namely:
\begin{itemize}
  \item \textsc{kl-UCB}$^+$: introduced in \cite{2016arXiv160101190K} and defined by the index\\ $u_a(t) = \sup\Big\{ q>\hat{\mu}_a(t):N_a(t) d(\hat{\mu}_a(t), q)\le \log(t(\log t)^c/N_a(t)) \Big\}$.
  \item \textsc{KL-Emp-UCB}: empirical \textsc{KL-UCB} introduced in \cite{2012arXiv1210.1136C} and using the empirical likelihood principle.
\end{itemize}

\subsection{Scenario 1: Bernoulli}

In the first scenario, the $K=5$ arms have Bernoulli distributions $\mathcal{B}(p_a)$ with respective parameters $(0.1, 0.3, 0.5, 0.5, 0.7)$ and thresholds $\tau_a$ in $(0.2, 0.2, 0.4, 0.6, 0.8)$.
Hence the profitable arms are the second and third ones. Notice that although arms $3$ and $4$ have the same distribution, namely $\mathcal{B}(0.5)$,
their thresholds are different and arm $4$ is non-profitable contrary to arm $4$.

\begin{figure}[H]
\vskip 0.2in
\begin{center}
\centerline{\includegraphics[width=\columnwidth]{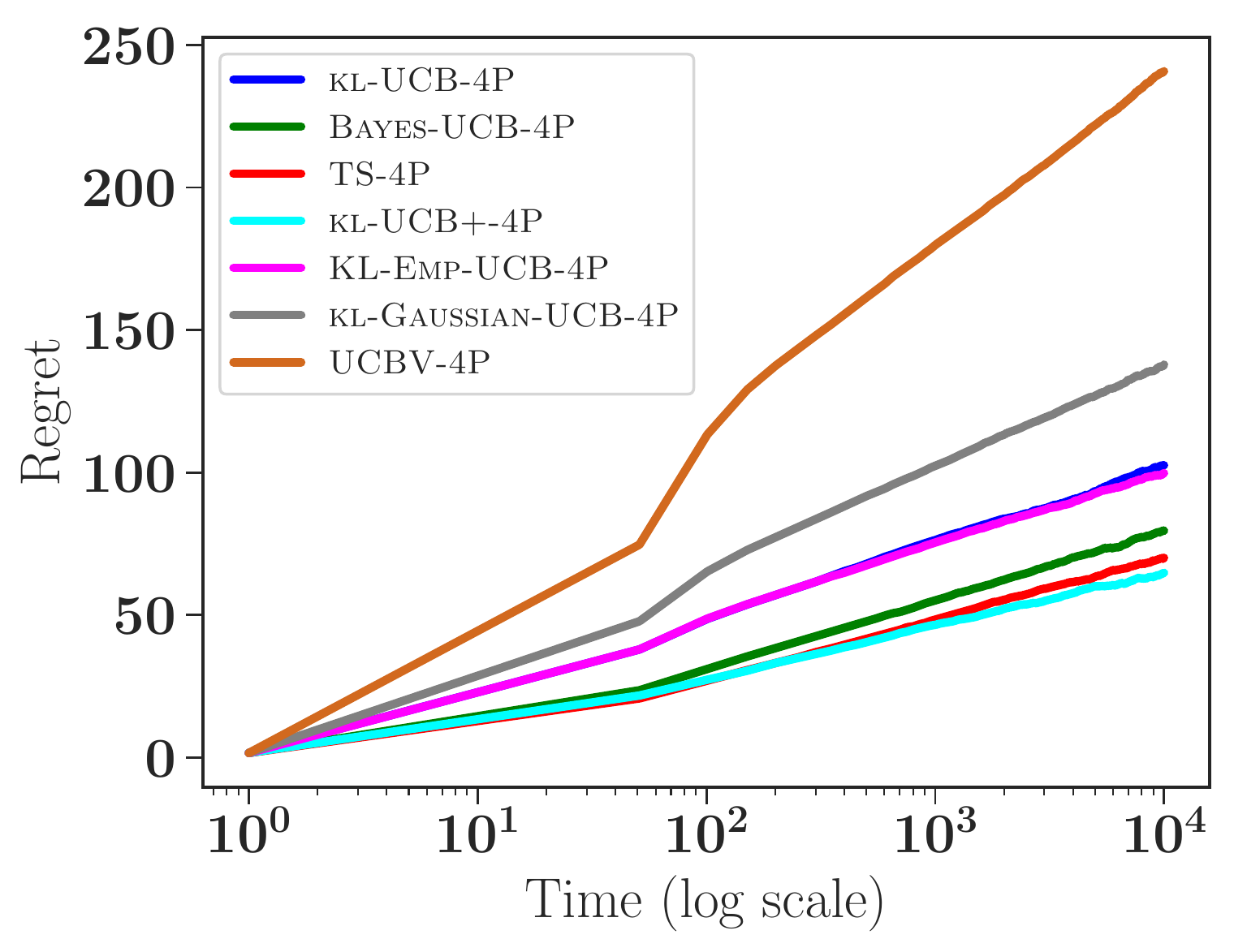}}
\caption{Regret of the various algorithms as a function of time in the Bernoulli scenario.}
\label{exp:bernoulli}
\end{center}
\vskip -0.2in
\end{figure}

All curves seem linear which means that the regret is regret logarithmically as a function of the time.
We observe that \textsc{kl-Gaussian-UCB-4P} has larger regret than other policies and that its slope seems larger too,
which confirms the observation resulting from Theorem \ref{th:upper_bound_N} followed by the application of Pinsker's inequality
that \textsc{kl-Gaussian-UCB-4P} performs worse than \textsc{kl-Bernoulli-UCB-4P}.
All other strategies have similar behavior and regret curves.

\subsection{Scenario 2: Poisson}
In the second scenario, the five categories $a\in\{1, \dots, 5\}$ have Poisson distributions $\mathcal{P}(\theta_a)$ with respective mean parameters $\theta_a$
as follows: $(1, 2, 3, 4, 5)$ and thresholds $\tau_a$: in $(2, 1, 4, 3, 6)$. In order to use \textsc{kl-Emp-UCB}, the Poisson rewards are truncated at
a maximal value chosen equal to $100$.
Hence the profitable arms are $2$ and $4$.

\begin{figure}[H]
\vskip 0.2in
\begin{center}
\centerline{\includegraphics[width=\columnwidth]{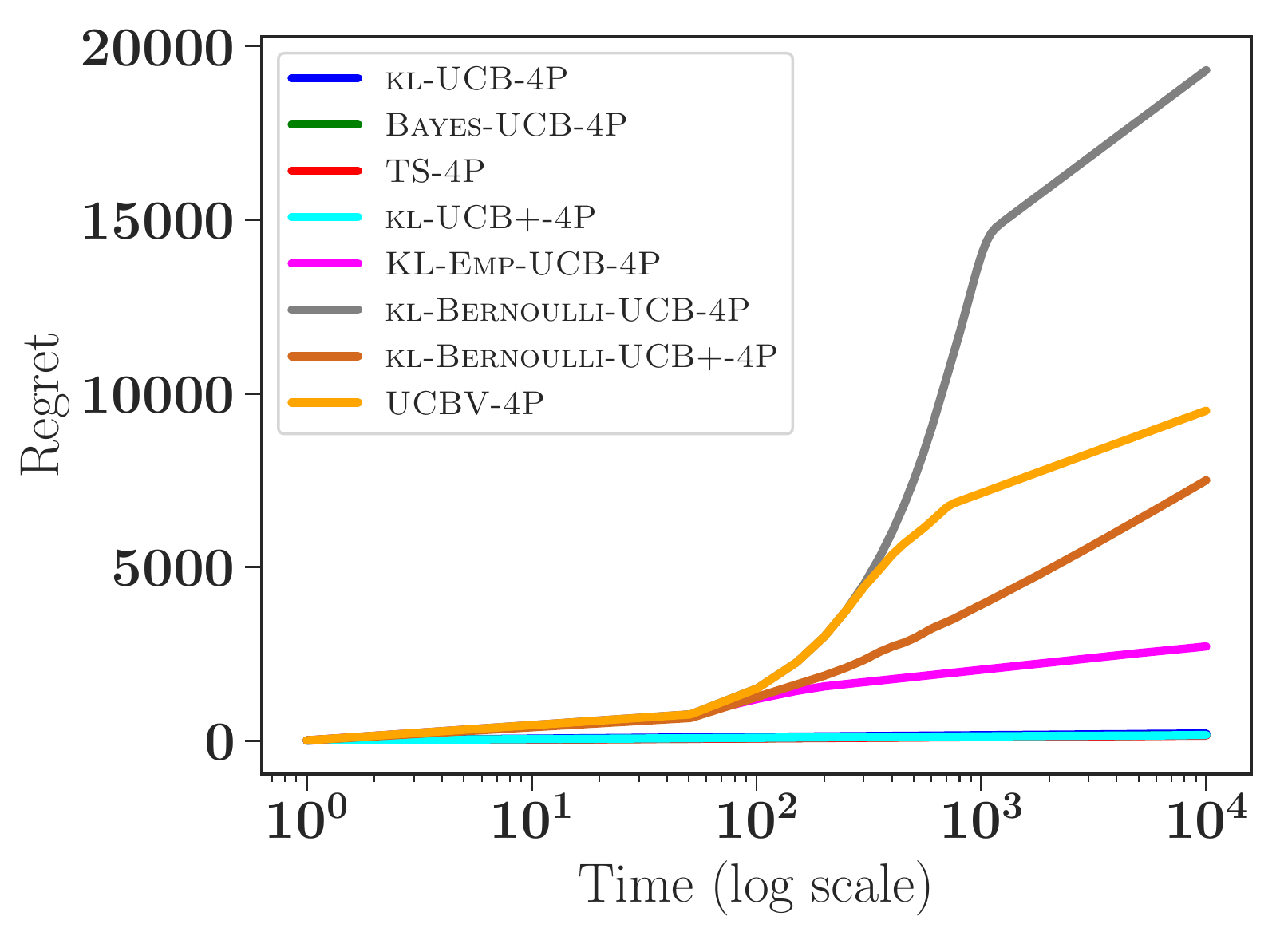}}
\caption{Regret of the various algorithms as a function of time in the Poisson scenario.}
\label{exp:poisson}
\end{center}
\vskip -0.2in
\end{figure}

Then we only run strategies performing the best which coincide with algorithms knowing in advance that the rewards follow Poisson distributions, through a well-suited prior distribution
for Bayesian policies or Kullback-Leibler divergence for UCB-like approaches. The distributions are kept the same but the problem is made harder
with sharper thresholds $\tau_a$: $(1.1, 1.9, 3.1, 3.9, 5.1)$.

\begin{figure}[H]
\vskip 0.2in
\begin{center}
\centerline{\includegraphics[width=\columnwidth]{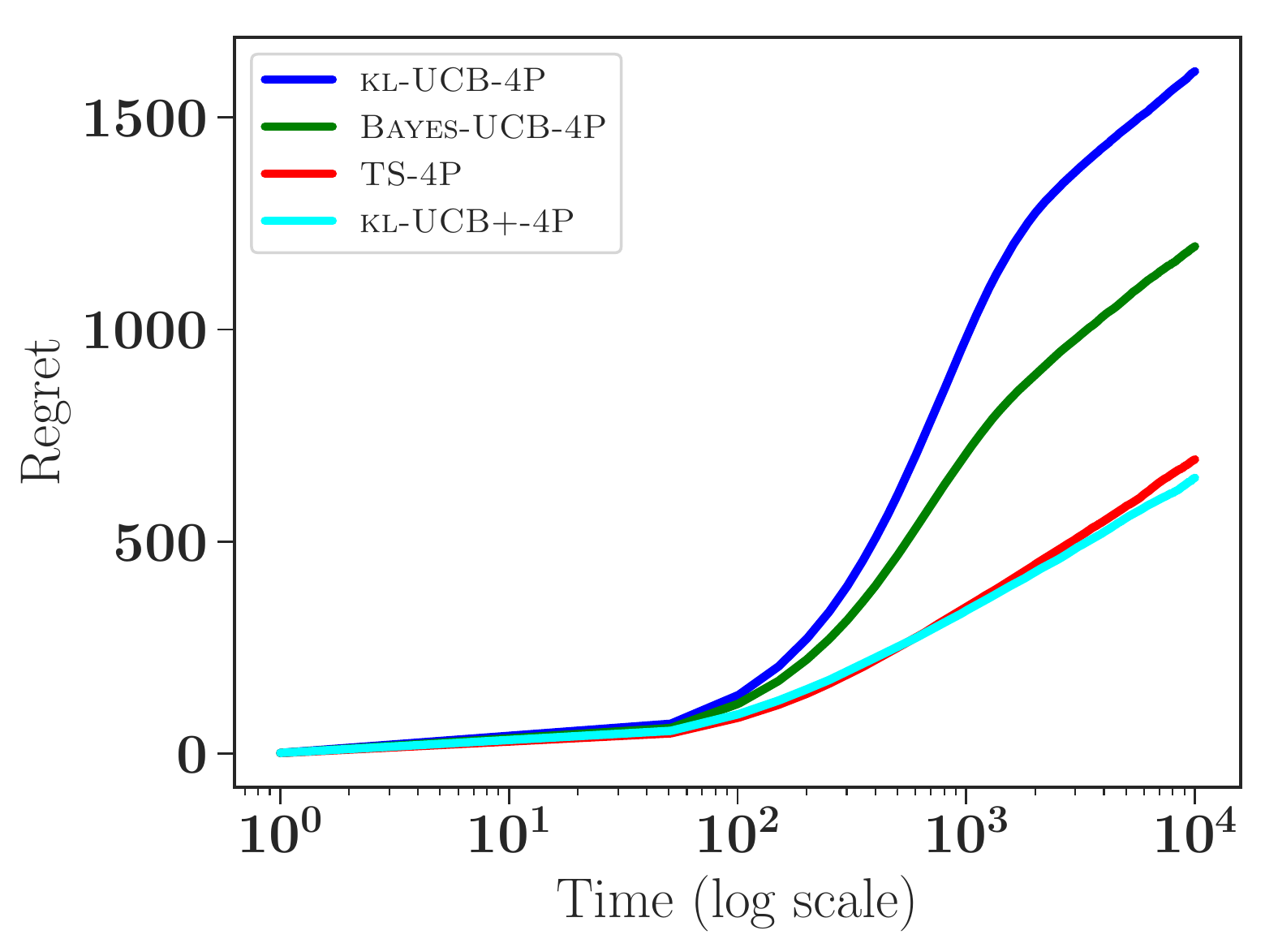}}
\caption{Regret of the best performing policies in the Poisson scenario.}
\label{exp:poisson_best}
\end{center}
\vskip -0.2in
\end{figure}

\subsection{Interpretation}

In both scenarios, all curves are linear i.e. the regret always grows logartihmically as a function of time.
We observe that the policies yielding smallest regret are those that know in advance the parametric family
where the distributions are living (Poisson in the second scenario).
Indeed, \textsc{kl-UCB-4P}, \textsc{Bayes-UCB-4P}, \textsc{TS-4P} and \textsc{kl-UCB}$^+$\textsc{-4P} all use the correct Kullback-Leibler divergence depending on the parametric family actually involved.
By contrast, the strategies achieving worse behavior are \textsc{kl-Gaussian-UCB-4P}, which always uses the Gaussian Kullback-Leibler divergence,
both \textsc{kl-Bernoulli-UCB-4P} and \textsc{kl-Bernoulli-UCB}$^+$\textsc{-4P} that always use the Kullback-Leibler divergence between Bernoulli distributions
and \textsc{Emp-KL-UCB-4P} which do only assume that the rewards are bounded.
Hence we see that prior knowledge on the reward distributions is critical in the efficiency of those algorithms.


\section{Conclusion}



Motivated by credit risk evaluation of different populations in a sequential context, this paper introduces the \emph{profitable bandit problem}, evaluates its difficulty by giving an asymptotic lower bound on the expected regret
and proposes and theoretically analyzes three algorithms, \textsc{kl-UCB-4P}, \textsc{Bayes-UCB-4P} and \textsc{TS-4P}, by giving finite-time upper bounds on their expected regret for reward distributions belonging to a one-dimensional exponential family.
All three algorithms are proven to be asymptotically optimal in the particular setting where for each catefory, a same number of clients is presented to the loaner at each time step.
An extension to general bounded distributions is proposed through two algorithms \textsc{kl-Bernoulli-UCB-4P} and \textsc{kl-Gaussian-UCB-4P} coming with finite-time analysis
directly derived from the analysis of \textsc{kl-UCB-4P}.
We finally compare all these strategies empirically and also against other policies inspired from other multi-armed bandits algorithms.
\textsc{Bayes-UCB-4P} and \textsc{TS-4P} perform the best in our numerical experiments and we observe that policies
having prior information on the distributions, through appropriate prior distribution for \textsc{Bayes-UCB-4P} and \textsc{TS-4P}
or Kullback-Leibler divergence for \textsc{kl-UCB-4P}, perform much better than non-adaptive strategies like
\textsc{kl-Bernoulli-UCB-4P} and \textsc{kl-Gaussian-UCB-4P}.


\bibliographystyle{plain}
\bibliography{bib}

\newpage

\appendix
\onecolumn

\section{Appendix}
\label{appendix}


\subsection{Proof of Theorem \ref{thm:lower_bound}}

  We use the inequality $(F)$ in Section $2$ in \cite{2016arXiv160207182G}, a consequence of the contraction of entropy property, which straightforwardly extends from the classical multi-armed bandit setting
  to ours where several arms can be pulled at each round $t$ and a number $C_a(t)\ge 1$ of observations are observed simultaneously for each pulled arm $a$. Then we have
  \begin{equation}
    \label{eq:contraction_entropy}
    \sum_{a=1}^K \mathbb{E}_\nu[N_a(T)]\text{KL}(\nu_a, \nu'_a) \ge \text{kl}(\mathbb{E}_\nu[Z], \mathbb{E}_{\nu'}[Z]),
  \end{equation}
  where $Z$ is any $\sigma(I_T)$-measurable random variable with values in $[0, 1]$.
  Consider a thresholding bandit problem $(\nu, \tau)\in \mathcal{D}$ with at least one non-profitable arm $a\in\{1, \dots, K\}$, we define a modified problem $(\nu', \tau)$ such that $\nu'_k=\nu_k$ for all $k\neq a$ and $\nu'_a \in \mathcal{D}_a$
  verifies $\mu'_a>\tau_a$. Then, considering $Z=N_a(T)/\tilde{C}_a(T)$, Eq. \ref{eq:contraction_entropy} rewrites as follows:

  \begin{equation*}
    \begin{split}
    \mathbb{E}_\nu[N_a(T)]\text{KL}(\nu_a, \nu'_a)
    &\ge \text{kl}(\mathbb{E}_\nu[N_a(T)]/\tilde{C}_a(T), \mathbb{E}_{\nu'}[N_a(T)]/\tilde{C}_a(T))\\
    &\ge \left(1-\frac{\mathbb{E}_\nu[N_a(T)]}{\tilde{C}_a(T)}\right) \log \left( \frac{\tilde{C}_a(T)}{\tilde{C}_a(T)-\mathbb{E}_{\nu'}[N_a(T)]} \right)-\log(2),
    \end{split}
  \end{equation*}

  where we used for the last inequality that for all $(p, q)\in [0, 1]^2$,
  \[
    \text{kl}(p, q)\ge (1-p)\log \left( \frac{1}{1-q} \right)-\log(2).
  \]
  Then, by uniform efficiency it holds: $\mathbb{E}_\nu[N_a(T)] = o(\tilde{C}_a(T))$ and $\tilde{C}_a(T)-\mathbb{E}_{\nu'}[N_a(T)] = o(\tilde{C}_a(T)^\alpha)$ for all $\alpha \in (0, 1]$.
  Hence for all $\alpha\in (0, 1]$,
  \begin{equation*}
    \liminf_{T\to \infty} \frac{1}{\log T}\mathbb{E}_\nu[N_a(T)] \text{KL}(\nu_a, \nu'_a)
    \ge \liminf_{T\to \infty} \frac{1}{\log T}\log \left( \frac{\tilde{C}_a(T)}{\tilde{C}_a(T)^\alpha} \right) = 1-\alpha.
  \end{equation*}
  Taking the limit $\alpha\to 0$ in the right-hand side and taking the infimum over all distributions $\nu'_a \in \mathcal{D}_a$ such that $\mu'_a>\tau_a$ in the left-hand side conclude the proof.

\subsection{Proof of Theorem \ref{th:upper_bound_N}}

  For any arm $a\in \{1, \dots, K\}$, the average reward at time $t$ is denoted by
  $\hat{\mu}_a(t)=S_a(t)/N_a(t)$ where $S_a(t)=\sum_{s=1}^t \sum_{c=1}^{C_a(s)} X_{a, c, s} \I\{a\in A_s \}$
  and $N_a(t)=\sum_{s=1}^t C_a(s) \I\{a\in A_s \}$.
  For every positive integer $s$, we also denote by $\hat{\mu}_{a, s}=(X_{a, 1}+\dots+X_{a, s})/s$ with $X_{a, 1}, \dots, X_{a, s}$
  the first $s$ samples pulled from arm $a$ (with arbitrary choice when some of these random variables are pulled together), so that $\hat{\mu}_t(a) = \hat{\mu}_{a, N_a(t)}$.
  The upper confidence bound for $\mu_a$ appearing in $\tau$-KL-UCB is then given by:
  \begin{equation*}
    u_a(t) = \sup\left\{ q>\hat{\mu}_a(t):N_a(t) d(\hat{\mu}_a(t), q)\le \log t+c\log\log t \right\}.
  \end{equation*}
  For $x, y \in [0, 1]$, define $d^+(x, y)=d(x, y)\mathbb{I}\{x<y\}$.

  (i). Let $a\in \{1, \dots, K\}\setminus \mathcal{A}^*$ be a non-profitable arm i.e. such that $\mu_a<\tau_a$. Given $\epsilon \in ]0, 1[$, we upper bound the expectation of $N_a(T)$ as follows,
  \begin{equation*}
    \mathbb{E}[N_a(T)] = \mathbb{E}\left[\sum_{t=1}^T C_a(t) \mathbb{I}\{a\in A_t\}\right]
    \le \tilde{c}_a^+ \E\left[\sum_{t=1}^T \mathbb{I}\{a\in A_t\}\right],
  \end{equation*}
  where $\tilde{c}_a^+ = \max_{1\le t\le T}\{\E[C_a(t)]\}$.
  Now observe for $t\ge 1$ that $a\in A_{t+1}$ implies $u_a(t)\ge\tau_a$ and hence,
  \begin{equation*}
    d^+(\hat{\mu}_a(t), \tau_a)\le d(\hat{\mu}_a(t), u_a(t)) = \frac{\log t + c\log\log t}{N_a(t)}.
  \end{equation*}
  Then,
  \begin{equation}
    \label{eq:split_subopt_klUCB}
    \begin{split}
      &\sum_{t=1}^T \mathbb{I}\{a\in A_t\}\\
      & = 1 + \sum_{t=1}^{T-1} \mathbb{I}\left\{a\in A_{t+1}\right\} \sum_{s=1}^{t} \sum_{1\le i_1 < \dots < i_s \le t} \mathbb{I}\left\{ a\in\bigcap_{i\in\{i_1, \dots, i_s\}} A_{i}, a\notin\bigcup_{i\in \{1, \dots, t\}\setminus\{i_1, \dots, i_s\}} A_{i}\right\}\\
      & \qquad \quad \times \mathbb{I}\left\{(C_a(i_1)+\dots+C_a(i_s)) d^+(\hat{\mu}_{a, C_a(i_1)+\dots+C_a(i_s)}, \tau_a)\le \log t+c\log\log t \right\}.
    \end{split}
  \end{equation}
  Given $\epsilon \in ]0, 1[$, we upper bound the last indicator function appearing in Eq. (\ref{eq:split_subopt_klUCB}) by
  \begin{equation}
    \label{eq:decouplage_klucb}
    \begin{split}
    &\mathbb{I}\{s<K_T\} + \sum_{k=c_a^- s}^{c_a^+ s} \mathbb{I}\left\{s\ge K_T, k d^+(\hat{\mu}_{a, k}, \tau_a) \le \log T+c\log\log T \right\}\\
    &\le \mathbb{I}\{s<K_T\} + \sum_{k=c_a^- s}^{c_a^+ s} \mathbb{I}\left\{s\ge K_T, d^+(\hat{\mu}_{a, k}, \tau_a)\le \frac{d(\mu_a, \tau_a)}{1+\epsilon} \right\},
  \end{split}
  \end{equation}
  where $K_T = \left\lceil (1+\epsilon)\frac{\log T+c\log\log T}{c_a^- d(\mu_a, \tau_a)} \right\rceil$.
  The last expression in Eq. (\ref{eq:decouplage_klucb}) is not using the indices $t, i_1, \dots, i_s$ which
  allows us to exchange the sums over $t$ and $s$ in Eq. (\ref{eq:split_subopt_klUCB}) and to obtain
  \begin{equation*}
    \begin{split}
      &\sum_{t=1}^T \mathbb{I}\{a\in A_t\}\\
      & \le 1 + \sum_{s=1}^{T} \left( \mathbb{I}\{s<K_T\} + \sum_{k=c_a^- s}^{c_a^+ s} \mathbb{I}\left\{s\ge K_T, d^+(\hat{\mu}_{a, k}, \tau_a)\le \frac{d(\mu_a, \tau_a)}{1+\epsilon} \right\} \right)\\
      & \qquad \quad \times \sum_{t=1}^{T-1} \mathbb{I}\left\{a\in A_{t+1}\right\} \sum_{1\le i_1 < \dots < i_s \le t} \mathbb{I}\left\{ a\in\bigcap_{i\in\{i_1, \dots, i_s\}} A_{i}, a\notin\bigcup_{i\in \{1, \dots, t\}\setminus\{i_1, \dots, i_s\}} A_{i}\right\}\\
      & \le K_T + \sum_{s=K_T}^{T} \sum_{k=c_a^- s}^{c_a^+ s} \mathbb{I}\left\{d^+(\hat{\mu}_{a, k}, \tau_a)\le \frac{d(\mu_a, \tau_a)}{1+\epsilon} \right\},
    \end{split}
  \end{equation*}
  where the last inequality is implied by
  \begin{equation}
    \label{eq:argument_classique_indicatrices}
    \sum_{t=1}^{T-1} \mathbb{I}\left\{a\in A_{t+1}\right\} \sum_{1\le i_1 < \dots < i_s \le t} \mathbb{I}\left\{ a\in\bigcap_{i\in\{i_1, \dots, i_s\}} A_{i}, a\notin\bigcup_{i\in \{1, \dots, t\}\setminus\{i_1, \dots, i_s\}} A_{i}\right\} \le 1.
  \end{equation}
  Hence,
  \begin{equation*}
    \begin{split}
    \mathbb{E}[N_a(T)] &\le \tilde{c}_a^+ \left\{ K_T + \sum_{s=K_T}^{+\infty} \sum_{k=c_a^- s}^{+\infty} \mathbb{P}\left( d^+(\hat{\mu}_{a, k}, \tau_a)\le \frac{d(\mu_a, \tau_a)}{1+\epsilon} \right) \right\}\\
    & \le (1+\epsilon) \frac{\tilde{c}_a^+}{c_a^-} \frac{\log T+c\log\log T}{d(\mu_a, \tau_a)}+ \tilde{c}_a^+ \left\{1+ \frac{H_1(\epsilon)}{T^{\beta_1(\epsilon)}}\right\},
  \end{split}
  \end{equation*}
  comes from Lemma \ref{lem:epsilon} with $H_1(\epsilon)$ and $\beta_1(\epsilon)$ positive functions of $\epsilon$.\\

  (ii). Now consider $a\in A^*$ i.e. verifying $\mu_a>\tau_a$. It follows,
  \begin{equation*}
    \tilde{C}_a(T)-\mathbb{E}[N_a(T)] = \mathbb{E}\left[\sum_{t=2}^T C_a(t) \mathbb{I}\{a\notin A_t\}\right]
    \le \tilde{c}_a^+ \sum_{t=1}^{T-1} \mathbb{P}\left( u_a(t) < \mu_a \right).
  \end{equation*}
  Let $t\in\{ 1, \dots, T-1 \}$ and observe that $(u_a(t) < \mu_a) \subset (d^+(\hat{\mu}_a(t), \mu_a) > d(\hat{\mu}_a(t), u_a(t))) $. Hence for $c\ge 3$ and $t\ge \max(3, c_a^+)$,
  \begin{equation*}
    \begin{split}
      &\mathbb{P}\left( u_a(t) < \mu_a \right)\\
      & \le \mathbb{P}\left( N_a(t) d^+(\hat{\mu}_a(t), \mu_a) > \delta_t \right)
      \le (\delta_t \log(c_a^+t) + 1)\exp(-\delta_t + 1)\\
      &= \frac{e((\log t)^2 + c\log(t)\log\log(t) + \log(c_a^+)\log(t) + c \log(c_a^+)\log\log(t) + 1)}{t(\log t)^c}\\
      &\le \frac{e(2c+3)}{t \log t},
    \end{split}
  \end{equation*}
  where $\delta_t = \log t + c\log\log t > 1$ and the second inequality results from the self-normalized concentration inequality
  stated in Lemma 10 in \cite{2012arXiv1210.1136C}.
  Then by summing over $t$,
  \begin{equation*}
    \begin{split}
    \tilde{C}_a(T)-\mathbb{E}[N_a(T)] &\le \tilde{c}_a^+\left\{ 2 + c_a^+ + e(2c+3)\sum_{t=3}^{T-1}\frac{1}{t\log t} \right\}\\
    &\le \tilde{c}_a^+ \{e(2c+3)\log\log T + c_a^+ + 3\}.
  \end{split}
  \end{equation*}


\subsection{Lemma \ref{lem:epsilon}}

\begin{lemma}
  \label{lem:epsilon}
  Let $a\in\{1, \dots, K\}\setminus\mathcal{A^*}$ a non-profitable arm (i.e. $\mu_a<\tau_a$), $\epsilon\in ]0, 1[$ and $K_T = \left\lceil f(\epsilon)\frac{\log T+c\log\log T}{c_a^- d(\mu_a, \tau_a)} \right\rceil$ with $f$ a function such that $f(\epsilon')>1$ for all $\epsilon' \in ]0, 1[$.
  Then there exist $H(\epsilon)>0$ and $\beta(\epsilon)>0$ such that
  \begin{equation*}
    \sum_{s=K_T}^{+\infty} \sum_{k=c_a^- s}^{+\infty} \mathbb{P}\left( d^+(\hat{\mu}_{a, k}, \tau_a) \le \frac{d(\mu_a, \tau_a)}{f(\epsilon)} \right)
    \le \frac{H(\epsilon)}{T^{\beta(\epsilon)}},
  \end{equation*}
  where $H(\epsilon)$ and $\beta(\epsilon)$ are positive functions of $\epsilon$ depeding on $\mu_a, \tau_a$ and $c_a^-$.
\end{lemma}

\begin{proof}
  Observe that $d^+(\hat{\mu}_{a, k}, \tau_a)\le d(\mu_a, \tau_a)/f(\epsilon)$ if and only if $\hat{\mu}_{a, k}\ge r(\epsilon)$ where $r(\epsilon)\in ]\mu_a, \tau_a[$
  verifies $d(r(\epsilon), \tau_a)=d(\mu_a, \tau_a)/f(\epsilon)$. Thus,
  \begin{equation*}
    \mathbb{P}\left( d^+(\hat{\mu}_{a, k}, \tau_a)\le \frac{d(\mu_a, \tau_a)}{f(\epsilon)} \right) = \mathbb{P}\left( \hat{\mu}_{a, k}\ge r(\epsilon) \right)
    \le e^{-k d(r(\epsilon), \mu_a)}
  \end{equation*}
  and
  \begin{equation*}
    \begin{split}
    \sum_{s=K_T}^{T} \sum_{k=c_a^- s}^{c_a^+ s} \mathbb{P}\left( d^+(\hat{\mu}_{a, k}, \tau_a)\le \frac{d(\mu_a, \tau_a)}{f(\epsilon)} \right)
    &\le \sum_{s=K_T}^{+\infty} \sum_{k=c_a^- s}^{+\infty} e^{-k d(r(\epsilon), \mu_a)}\\
    &= \frac{1}{1-e^{-d(r(\epsilon), \mu_a)}} \sum_{s=K_T}^{+\infty} e^{-c_a^- s d(r(\epsilon), \mu_a)}\\
    &= \frac{e^{-c_a^- d(r(\epsilon), \mu_a)K_T}}{(1-e^{-d(r(\epsilon), \mu_a)})\left(1-e^{- c_a^- d(r(\epsilon), \mu_a)}\right)}\\
    &\le \frac{H(\epsilon)}{T^{\beta(\epsilon)}},
  \end{split}
  \end{equation*}
  where $H(\epsilon)=\left[\left(1-e^{-d(r(\epsilon), \mu_a)}\right)\left(1-e^{- c_a^- d(r(\epsilon), \mu_a)}\right)\right]^{-1}$ and $\beta(\epsilon)=f(\epsilon)d(r(\epsilon), \mu_a)/d(\mu_a, \tau_a)$.
\end{proof}

\subsection{Proof of Theorem \ref{th:upper_bound_N_bayesUCB}}


  (i). Let $a\in \{1, \dots, K\}\setminus \mathcal{A}^*$ be a non-profitable arm (i.e. $\mu_a<\tau_a$). We upper bound the expectation of $N_a(T)$ as follows:

  \begin{alignat}{2}
    \label{eq:split_subopt_BayesUCB}
    &\mathbb{E}[N_a(T)] = \mathbb{E}\left[\sum_{t=1}^T C_a(t) \mathbb{I}\{a\in A_t\}\right]
    \le \tilde{c}_a^+ \E \left[ 1 + \mathbb{E}\left[\sum_{t=1}^{T-1} \mathbb{I}\{ \bar{q}_a(t) \ge \tau_a \}\right] \right] \nonumber \\
    &= \tilde{c}_a^+ \E \left[ 1 + \sum_{t=1}^{T-1} \I\left\{\pi_{a, N_a(t), \bar{\mu}_a(t)}([\tau_a, \mu^+[) \ge \frac{1}{t(\log t)^c}, a \in A_{t+1}\right\} \right] \nonumber \\
    &\le \tilde{c}_a^+ \E \Biggl[ 1 + \sum_{t=1}^{T-1} \I\left\{\bar{\mu}_a(t)<\tau_a, \pi_{a, N_a(t), \bar{\mu}_a(t)}([\tau_a, \mu^+[) \ge \frac{1}{t(\log t)^c}, a \in A_{t+1} \right\}\\
    &\qquad \quad \quad + \sum_{t=1}^{T-1} \I\{\bar{\mu}_a(t) \ge \tau_a, a \in A_{t+1}\} \Biggr].
  \end{alignat}

  Using Lemma 4 in \cite{2016arXiv160101190K}, the first sum in (\ref{eq:split_subopt_BayesUCB}) is upper bounded by
  \begin{equation}
    \begin{split}
      \label{eq:first_sum}
    & \sum_{t=1}^{T-1} \I\left\{ B\sqrt{N_a(t)} e^{-N_a(t)d^+(\bar{\mu}_a(t), \tau_a)} \ge \frac{1}{t(\log t)^c}, a \in A_{t+1}\right\} \\
    & = \sum_{t=1}^{T-1} \mathbb{I}\left\{a\in A_{t+1}\right\} \sum_{s=1}^{t} \sum_{1\le i_1 < \dots < i_s \le t} \mathbb{I}\left\{ a\in\bigcap_{i\in\{i_1, \dots, i_s\}} A_{i}, a\notin\bigcup_{i\in \{1, \dots, t\}\setminus\{i_1, \dots, i_s\}} A_{i}\right\}\\
    & \quad \times \I\left\{B\sqrt{C_a(i_1)+\dots+C_a(i_s)} e^{-(C_a(i_1)+\dots+C_a(i_s))d^+(\bar{\mu}_{a, C_a(i_1)+\dots+C_a(i_s)}, \tau_a)} \ge \frac{1}{t(\log t)^c}\right\},
  \end{split}
  \end{equation}
  where $B$ is a constant depending on $\mu_0^-$, $\mu_0^+$ and on prior densities.
  Then we upper bound the last indicator function appearing in Eq. (\ref{eq:first_sum}) by
  \begin{equation}
    \label{eq:bayesUCB_sumk}
    \begin{split}
      & \I\{ s<K_T \} + \sum_{k=c_a^-s}^{c_a^+s} \I\left\{s\ge K_T, k d^+(\bar{\mu}_{a, k}, \tau_a) \le \log T + c\log\log T + \frac{1}{2}\log k + \log B \right\} \\
      &\le \I\{ s<K_T \} + \sum_{k=c_a^-s}^{c_a^+s} \I\left\{s\ge K_T, k d^+(\hat{\mu}_{a, k}, \tau_a) \le \log T + c\log\log T + \frac{1}{2}\log k + \log B \right\}\\
      &\quad+ \I\{\hat{\mu}_{a, k}<\mu_0^-\}.
    \end{split}
  \end{equation}
  We are now able to upper bound the right-hand side expression in Eq. (\ref{eq:first_sum}) by injecting Eq. (\ref{eq:bayesUCB_sumk}) and switching the sums on indices $t$ and $s$, which leads to
  \begin{equation}
    \label{eq:bayesUCB_sumt_removed}
    \begin{split}
      &\sum_{t=1}^{T-1} \I\left\{\bar{\mu}_a(t)<\tau_a, \pi_{a, N_a(t), \bar{\mu}_a(t)}([\tau_a, \mu^+[) \ge \frac{1}{t(\log t)^c}, a \in A_{t+1} \right\}\\
      &\le K_T-1 + \sum_{s=1}^T \sum_{k=c_a^-s}^{c_a^+s} \I\left\{s\ge K_T, k d^+(\hat{\mu}_{a, k}, \tau_a) \le \log T + c\log\log T + \frac{1}{2}\log k + \log B \right\}\\
      &\quad+ \I\{\hat{\mu}_{a, k}<\mu_0^-\},
    \end{split}
  \end{equation}
  where we used the same argument as in Eq. (\ref{eq:argument_classique_indicatrices}) to get rid of the sum over $t$.

  Given $\epsilon \in ]0, 1[$ we define $K_T = \left\lceil \frac{1+\epsilon}{1-\epsilon}\frac{\log T + c\log\log T}{c_a^- d(\mu_a, \tau_a)} \right\rceil$
  and denote by $N_a(\epsilon)$ the constant such that
  $T\ge N_a(\epsilon)$ implies:
  \begin{equation}
    K_T \ge \left\lceil \frac{3}{c_a^-} \right\rceil \quad \text{and} \quad  \frac{1}{c_a^- K_T} \left( \frac{1}{2}\log (c_a^- K_T) + \log(B) \right) \le \frac{\epsilon}{1+\epsilon}d(\mu_a, \tau_a),
  \end{equation}
  where the first inequality ensures that for all $k\ge c_a^- K_T$, the function $k\mapsto \log(x)/x$ decreases.
  Hence, the first indicator function appearing in the right-hand side in Eq. (\ref{eq:bayesUCB_sumt_removed}) is upper bounded by
  \begin{equation}
    \label{eq:bayes_split_kt}
    \I\left\{ s\ge K_T, d^+(\hat{\mu}_{a, k}, \tau_a)\le \frac{1-\epsilon}{1+\epsilon}d(\mu_a, \tau_a) \right\}.
  \end{equation}
  By combining equations (\ref{eq:split_subopt_BayesUCB}), (\ref{eq:bayesUCB_sumt_removed}) and (\ref{eq:bayes_split_kt}) we obtain
  \begin{equation}
    \label{eq:bayes_passage_proba}
    \begin{split}
    \E[N_a(T)]
    & \le \tilde{c}_a^+\Biggl\{ K_T + \sum_{s=K_T}^T \sum_{k=c_a^- s}^{c_a^+ s} \mathbb{P}\left( d^+(\hat{\mu}_{a, k}, \tau_a)\le \frac{1-\epsilon}{1+\epsilon}d(\mu_a, \tau_a) \right)\\
    & \qquad + \sum_{s=1}^T \sum_{k=c_a^- s}^{c_a^+ s} \mathbb{P}(\hat{\mu}_{a, k}<\mu_0^-)
    + \sum_{t=1}^{T-1} \mathbb{P}(\bar{\mu}_a(t) \ge \tau_a, a \in A_{t+1}) \Biggr\},
  \end{split}
  \end{equation}
  where the first sum can be upper bounded by $H_3(\epsilon) T^{-\beta_2(\epsilon)}$ with $H_3(\epsilon)>0$ and $\beta_2(\epsilon)>0$ thanks to Lemma \ref{lem:epsilon}.
  We upper bound the second sum in Eq. \ref{eq:bayes_passage_proba} with Chernoff inequality:
  \begin{equation*}
    \begin{split}
    \sum_{s=1}^T \sum_{k=c_a^- s}^{c_a^+ s} \mathbb{P}(\hat{\mu}_{a, k}<\mu_0^-)
    &\le \sum_{s=1}^{+\infty} \sum_{k=c_a^- s}^{+\infty} e^{-k d(\mu_0^-, \mu_a)}\\
    &= \frac{e^{-c_a^- d(\mu_0^-, \mu_a)}}{\left(1-e^{-d(\mu_0^-, \mu_a)}\right)\left(1-e^{- c_a^- d(\mu_0^-, \mu_a)}\right)}.
  \end{split}
  \end{equation*}
  Finally, we upper bound the third sum in Eq. (\ref{eq:bayes_passage_proba}) by
  \begin{equation}
    \begin{split}
    &\E\left[\sum_{t=1}^{T-1} \I\{\hat{\mu}_{a, s} \ge \tau_a, a \in A_{t+1}\}\right]\\
    & \le \E\Biggl[ \sum_{t=1}^{T-1} \mathbb{I}\left\{a\in A_{t+1}\right\} \sum_{s=1}^{t} \sum_{1\le i_1 < \dots < i_s \le t} \mathbb{I}\left\{ a\in\bigcap_{i\in\{i_1, \dots, i_s\}} A_{i}, a\notin\bigcup_{i\in \{1, \dots, t\}\setminus\{i_1, \dots, i_s\}} A_{i}\right\}\\
    &\qquad \qquad \qquad \qquad \qquad \qquad \qquad \qquad \quad \times \I\left\{\hat{\mu}_{a, C_a(i_1)+\dots+C_a(i_s)} \ge \tau_a \right\}\Biggr]\\
    & \le \sum_{s=1}^{T} \sum_{k=c_a^- s}^{c_a^+s} \mathbb{P}(\hat{\mu}_{a, k} \ge \tau_a )
    \le \frac{e^{-c_a^- d(\tau_a, \mu_a)}}{\left(1-e^{-d(\tau_a, \mu_a)}\right)\left(1-e^{- c_a^- d(\tau_a, \mu_a)}\right)}.
  \end{split}
  \end{equation}
  where we respectively used Eq. (\ref{eq:argument_classique_indicatrices}) and Chernoff inequality in the two last inequalities.

  (ii). Now consider $a\in A^*$. We have,
  \begin{align}
    & \tilde{C}_a(T)-\E[N_a(T)] = \E\left[\sum_{t=1}^{T-1} C_a(t+1) \mathbb{I}\{a\notin A_{t+1}\}\right]
    = \tilde{c}_a^+ \sum_{t=1}^{T-1} \mathbb{P}(\bar{q}_a(t) < \tau_a) \nonumber \\
    &\le \tilde{c}_a^+ \left\{t_0 - 1 + \sum_{t=t_0}^{T-1} \mathbb{P}\left(\hat{\mu}_a(t) < \tau_a, N_a(t)\ge (\log t)^2\right)
    + \sum_{t=1}^{T-1} \mathbb{P}\left(\bar{q}_a(t) < \tau_a, N_a(t)\le (\log t)^2\right)\right\} \label{eq:split_opt_BayesUCB},
  \end{align}
  where $t_0=\max(t_1, t_2)$ with $t_1$ the smallest integer verifying $C^2 t_0 (\log t_0)^{2c}\ge 1$, which implies for all $t\ge t_1$ that $\bar{\mu}_a(t) \le \bar{q}_a(t)$,
  and $t_2 = \left\lceil \exp( 2/d(\tau_a, \mu_a) ) \right\rceil$ to ensure that $d(\tau_a, \mu_a)(\log t)^2 \ge 2\log t$ for all $t\ge t_2$.
  To upper bound the first sum in Eq. (\ref{eq:split_opt_BayesUCB}) we write for $t\ge t_0$,
  \begin{equation*}
    \begin{split}
    \mathbb{P}\left(\hat{\mu}_a(t) < \tau_a, N_a(t)\ge (\log t)^2\right)
    &\le \sum_{s=\lceil (\log t)^2 \rceil}^t \mathbb{P}(\hat{\mu}_{a, s} < \tau_a)
    \le \sum_{s=\lceil (\log t)^2 \rceil}^{+\infty} e^{-sd(\tau_a, \mu_a)}\\
    &\le e^{-d(\tau_a, \mu_a) (\log t)^2} \le \frac{1}{t^2}.
    \end{split}
  \end{equation*}

  To upper bound the second sum in Eq. (\ref{eq:split_opt_BayesUCB}) use again Lemma 4 in \cite{2016arXiv160101190K},
  \begin{equation*}
    \begin{split}
    &\mathbb{P}\left(\bar{q}_a(t) < \tau_a, N_a(t)\le (\log t)^2\right)
    = \mathbb{P}\left(\pi_{a, N_a(t), \bar{\mu}_a(t)}([\tau_a, \mu^+[)<\frac{1}{t(\log t)^c}, N_a(t)\le (\log t)^2\right)\\
    &\le \mathbb{P}\left( \frac{A e^{-N_a(t)d(\bar{\mu}_a(t), \tau_a)}}{N_a(t)} <\frac{1}{t(\log t)^c}, N_a(t)\le (\log t)^2\right)\\
    &= \mathbb{P}\left( N_a(t)d^+(\hat{\mu}_a(t), \tau_a) > \log\left(\frac{A t(\log t)^c}{N_a(t)}\right), N_a(t)\le (\log t)^2\right)\\
    &\le \mathbb{P}( N_a(t)d^+(\hat{\mu}_a(t), \tau_a) > \log(At) + (c-2)\log\log t),
    \end{split}
  \end{equation*}
  where $A$ is a constant depending on $\mu_0^-$, $\mu_0^+$ and on prior densities.
  Then for $c\ge 5$, using the self-normalized deviation inequality stated in Lemma 10 in \cite{2012arXiv1210.1136C},
  we have,
  \begin{equation*}
    \begin{split}
      &\mathbb{P}( N_a(t)d^+(\hat{\mu}_a(t), \tau_a) > \log(At) + (c-2)\log\log t) \le (\delta_t \log(c_a^+ t) + 1)\exp(-\delta_t + 1)\\
      &= \frac{e((\log(t))^2 + (c-2)\log(t)\log\log(t) + \log(A c_a^+)\log(t)+ (c-2)\log(c_a^+)\log\log(t) + \log(A)\log(c_a^+)+1)}{At(\log(t))^{c-2}}\\
      &\le \frac{e(2(c-2)+4)}{A t\log(t)},
    \end{split}
  \end{equation*}
  where we assumed $t\ge t_a = \max(e/A, 3, A, c_a^+, Ac_a^+)$ to ensure the last inequality and that $\delta_t=\log(At) + (c-2)\log\log(t) > 1$.
  Then by summing over $t$,
  \begin{equation*}
    \begin{split}
      \tilde{C}_a(T)-\mathbb{E}[N_a(T)]
      &\le \tilde{c}_a^+\left\{ t_a + \frac{e(2(c-2)+4)}{A}\sum_{t=3}^{T-1}\frac{1}{t\log t} \right\}\\
      &\le \tilde{c}_a^+ \{e(2(c-2)+4)\log\log T + t_a + 1\}.
    \end{split}
  \end{equation*}

\subsection{Proof of Theorem \ref{th:upper_bound_N_thompson}}


We first introduce some notations.
Let $L(\theta)=(1/2)\min(1, \sup_x p(x | \theta))$ and for any $\delta_a > 0$,
\begin{equation*}
  E_{a, s} = \Biggl( \exists s' \in \{ 1, \dots, s \}, p(X_{a, s'} | \theta_a)\ge L(\theta_a),
  \left| \frac{\sum_{u=1, u\neq s'}^{s} X_{a, u}}{s-1} - \mu_a \right| \le \delta_a \Biggr)
\end{equation*}
is an event where there is at least one 'likely' observation of arm $a$ (namely $X_{a, s'}$) and such that the empirical sufficient statistic is close to its true mean.
We also define $E_a(t) = E_{a, N_a(t)}$.
\begin{remark}
In the definition of $E_{a, s}$, the 'likely' observation $X_{a, s'}$ is only needed for technical reasons
when the Jeffreys prior $\pi_{a, 0}$ is improper (see Remark 8 in \cite{korda2013thompson} for more details).
\end{remark}

We now recall the Theorem 4 in \cite{korda2013thompson}, an important result on the posterior concentration under the event $E_a(t)$.
\begin{lemma}
  \label{lem:posterior_concentration}
  There exists problem-dependent constants $C_{1, a}$ and $N_{1, a}$ and a function $\Delta \mapsto C_{2, a}(\Delta)$
  such that for $\delta_a \in ]0, 1[$ and $\Delta>0$ verifying $1-\delta_a C_{2, a}(\Delta) > 0$, it holds whenever $N_a(t)\ge N_{1, a}$:
  \begin{equation*}
    \mathbb{P}(\mu(\theta_a(t)) \ge \mu_a + \Delta, E_a(t) | (X_{a, s})_{1\le s\le N_a(t)})
    \le C_{1, a} N_a(t) e^{-(N_a(t)-1)(1-\delta_a C_{2, a}(\Delta))d(\mu_a, \mu_a+\Delta)}
  \end{equation*}
  and
  \begin{equation*}
    \mathbb{P}(\mu(\theta_a(t)) \le \mu_a - \Delta, E_a(t) | (X_{a, s})_{1\le s\le N_a(t)})
    \le C_{1, a} N_a(t) e^{-(N_a(t)-1)(1-\delta_a C_{2, a}(\Delta))d(\mu_a, \mu_a-\Delta)}.
  \end{equation*}
\end{lemma}

Thanks to these concentration inequalities we can derive bounds on the expected number of pulls
of any arm.

For all arms $a\in \{1, \dots, K\}$ and $t\ge 2$, $\theta_a(t)$ is a r.v. sampled from the posterior distribution $\pi_a(t)$ on $\theta_a$ obtained after $N_a(t-1)$ observations.
For all $s\ge 1$, we also denote by $\theta_{a, s}$ a r.v. sampled from the posterior distribution resulting from the first $s$ observations pulled from arm $a$ (with arbitrary choice when some of these random variables are pulled together), so that $\theta_a(t) = \theta_{a, N_a(t-1)}$.

We now prove Theorem \ref{th:upper_bound_N_thompson}.

  (i). Let $a\in \{1, \dots, K\}\setminus \mathcal{A}^*$ be a non-profitable arm (i.e. $\mu_a<\tau_a$). We upper bound the expectation of $N_a(T)$ as follows:

  \begin{equation}
    \label{eq:split_subopt_thompson}
    \mathbb{E}[N_a(T)] = \E\left[ C_a(t) \sum_{t=1}^T \I\{a\in A_t\} \right]
    \le \tilde{c}_a^+ \left\{ 1 + \sum_{t=2}^{T} \mathbb{P}\left(a\in A_{t}, E_a(t)\right) + \mathbb{P}\left(a\in A_{t}, E_a(t)^c\right)\right\}.
  \end{equation}

  First observe that the first sum in the right-hand side in Eq. (\ref{eq:split_subopt_thompson}) is equal to
  \begin{equation*}
    \begin{split}
      &\E\Biggl[ \sum_{t=2}^{T} \mathbb{I}\left\{a\in A_{t}\right\} \sum_{s=1}^{t-1} \sum_{1\le i_1 < \dots < i_s \le t-1} \mathbb{I}\left\{ a\in\bigcap_{i\in\{i_1, \dots, i_s\}} A_{i}, a\notin\bigcup_{i\in \{1, \dots, t-1\}\setminus\{i_1, \dots, i_s\}} A_{i}\right\}\\
      &\quad \times \I\left\{\mu(\theta_{a, C_a(i_1)+\dots+C_a(i_s)}) \ge \tau_a, E_{a, C_a(i_1)+\dots+C_a(i_s)}\right\}\Biggr].
    \end{split}
  \end{equation*}

  Then, given $\epsilon \in ]0, 1[$, by choosing $\delta_a \le \epsilon/C_{2, a}(|\Delta_a|)$, defining $K_T = \left\lceil \frac{1+\epsilon}{1-\epsilon}\frac{\log T}{c_a^- d(\mu_a, \tau_a)} \right\rceil$
  and observing that $\I\{\mu(\theta_{a, C_a(i_1)+\dots+C_a(i_s)}) \ge \tau_a, E_{a, C_a(i_1)+\dots+C_a(i_s)}\} \le \I\{s<K_T\} + \sum_{k=c_a^- s}^{c_a^+ s} \I\{s\ge K_T, \mu(\theta_{a, k}) \ge \tau_a, E_{a, k}\}$, we obtain

  \begin{equation*}
    \begin{split}
      \sum_{t=2}^{T} \mathbb{P}\left(a\in A_{t}, E_a(t)\right)
      &\le K_T-2 + \sum_{s=K_T}^{T} \sum_{k=c_a^- s}^{c_a^+ s} \mathbb{P}\left(\mu(\theta_{a, k}) \ge \tau_a, E_{a, k}\right)\\
      &\le K_T-2 + \sum_{s=K_T}^T \sum_{k=c_a^- s}^{c_a^+ s} C_{1, a} k e^{-(k - 1)(1-\epsilon)d(\mu_a, \tau_a)}\\
      &\le \frac{1+\epsilon}{1-\epsilon}\frac{\log T}{c_a^- d(\mu_a, \tau_a)} + C_{1, a} T (c_a^+ K_T)^2 e^{-(c_a^- K_T - 1)(1-\epsilon)d(\mu_a, \tau_a)}\\
      &\le \frac{1+\epsilon}{1-\epsilon}\frac{\log T}{c_a^- d(\mu_a, \tau_a)} + C_{1, a} e^{(1-\epsilon) d(\mu_a, \tau_a)} \frac{(c_a^+ K_T)^2}{T^\epsilon},
    \end{split}
  \end{equation*}
  where we used in the first inequality Eq. (\ref{eq:argument_classique_indicatrices}).
  In the second and third inequalities we assumed $T$ larger than $N_a(\epsilon)$ verifying $T\ge N_a(\epsilon) \Rightarrow K_T \ge max(N_{1, a}/c_a^-, N_{2, a})$
  with $N_{1, a}$ defined in Lemma \ref{lem:posterior_concentration} and $N_{2, a}$ such that the function $u \mapsto u^2 e^{-(c_a^- u-1)(1-\epsilon)d(\mu_a, \tau_a)}$
  is decreasing for $u \ge N_{2, a}$.

  In order to upper bound the second sum in the right-hand side in Eq. (\ref{eq:split_subopt_thompson}) we first introduce the following events:
  \begin{equation*}
    B_{a, s} = (\forall s' \in \{1, \dots, s\}, p(X_{a, s'}|\theta_a)\le L(\theta_a))
  \end{equation*}
  and
  \begin{equation*}
    D_{a, s} = \left(\exists s' \in \{1, \dots, s\}, \left| \frac{\sum_{u=1, u\neq s'}^s X_{a, u}}{s-1} - \mu_a \right| > \delta_a \right).
  \end{equation*}
  Then observing that $E_a(t)^c \subset B_{a, N_a(t)} \bigcup D_{a, N_a(t)}$ and it holds
  \begin{equation*}
    \begin{split}
    &\sum_{t=2}^{T} \mathbb{P}(a\in A_{t}, E_a(t)^c)\\
    &\le \E\Biggl[ \sum_{t=2}^{T} \mathbb{I}\left\{a\in A_t \right\} \sum_{s=1}^{t-1} \sum_{1\le i_1 < \dots < i_s \le t-1} \mathbb{I}\left\{ a\in\bigcap_{i\in\{i_1, \dots, i_s\}} A_{i}, a\notin\bigcup_{i\in \{1, \dots, t-1\}\setminus\{i_1, \dots, i_s\}} A_{i}\right\}\\
    &\qquad\qquad\qquad\qquad\qquad\qquad\qquad\qquad\quad \times\sum_{k=c_a^- s}^{c_a^+ s} \I\left\{B_{a, k}\right\} + \I\left\{D_{a, k}\right\}\Biggr]\\
    & \le \sum_{s=1}^T \sum_{k=c_a^- s}^{c_a^+ s} \mathbb{P}\left(B_{a, k}\right) + \mathbb{P}\left(D_{a, k}\right)\\
    &\le \sum_{s=1}^{+\infty} c_a^+ s \mathbb{P}(p(X_{a, 1} | \theta_a) < L(\theta_a))^{c_a^- s} + (c_a^+ s)^2 (e^{-(c_a^- s-1)d(\mu_a-\delta_a, \mu_a)}+e^{-(c_a^- s-1)d(\mu_a+\delta_a, \mu_a)}) < + \infty,
    \end{split}
  \end{equation*}
  where we used Eq. (\ref{eq:argument_classique_indicatrices}) in the second inequality.

  (ii). Now consider $a\in A^*$ i.e. verifying $\mu_a>\tau_a$. We have
\begin{alignat}{2}
    &\tilde{C}_a(T)-\mathbb{E}[N_a(T)] = \E\left[ \sum_{t=2}^{T} C_a(t) \I\{a \notin A_{t}\} \right]
    \le \tilde{c}_a^+ \sum_{t=2}^{T} \mathbb{P}(\mu(\theta_a(t))<\tau_a) \nonumber \\
    &\le \tilde{c}_a^+ \left\{ \sum_{t=2}^{T} \mathbb{P}(\mu(\theta_a(t))<\tau_a, E_a(t) | N_a(t)>t^b)
    + \sum_{t=2}^{T} \mathbb{P}(E_a(t)^c | N_a(t)>t^b) + \sum_{t=1}^{+\infty} \mathbb{P}(N_a(t)\le t^b) \right\} \label{eq:split_opt_thompson}.
\end{alignat}

By applying Lemma \ref{lem:posterior_concentration}, the first sum in Eq. (\ref{eq:split_opt_thompson}) is upper bounded by
\begin{equation*}
  N_{0, a}^{1/b} + \sum_{t=\left\lceil N_{0, a}^{1/b} \right\rceil}^{+\infty} C_{1, a} t^b e^{-(t^b - 1)(1-\delta_a C_{2, a}(|\Delta_a|))d(\mu_a, \tau_a)} < + \infty,
\end{equation*}
where $N_{0, a} = \max(N_{1, a}, N_{3, a})$ with $N_{3, a}$ such that the function $u \mapsto u e^{-(u-1)(1-\delta_a C_{2, a}(|\Delta_a|))d(\mu_a, \tau_a)}$
is decreasing for $u \ge N_{3, a}$.

By applying Chernoff inequality we upper bound the second sum in Eq. (\ref{eq:split_opt_thompson}) by
\begin{equation*}
  \begin{split}
    &\sum_{t=2}^{T} \mathbb{P}(E_a(t)^c | N_a(t)>t^b)
    \le \sum_{t=2}^{T} \sum_{s=\lceil t^b/c_a^+ \rceil}^t \sum_{k=c_a^- s}^{c_a^+ s} \mathbb{P}(B_{a, k}) + \mathbb{P}(D_{a, k})\\
    &\le \sum_{t=1}^{+\infty} c_a^+ t^2 \mathbb{P}(p(X_{a, 1}|\theta_a)\le L(\theta_a))^{\frac{c_a^-}{c_a^+} t^b} + 2 (c_a^+)^2 t^3 \left(e^{-\left(\frac{c_a^-}{c_a^+} t^b -1\right) d(\mu_a-\delta_a, \mu_a)} + e^{-\left(\frac{c_a^-}{c_a^+} t^b -1\right) d(\mu_a+\delta_a, \mu_a)}\right) < + \infty.
  \end{split}
\end{equation*}

Finally we upper bound the third sum in Eq. (\ref{eq:split_opt_thompson}) with the following result, inspired from Proposition 5 in \cite{korda2013thompson}.
In our case its proof is simpler as there are no dependencies between arms in the profitable bandit problem.

\begin{lemma}
  \label{lem:N_larger_tb}
  For any profitable arm $a\in A^*$ and any $b\in ]0, 1[$, there exists a problem-dependent constant $C_b < +\infty$ such that
  \begin{equation*}
    \sum_{t=1}^{+\infty} \mathbb{P}(N_a(t)\le t^b) \le C_b.
  \end{equation*}
\end{lemma}

Then, by using the Bernstein-Von-Mises theorem telling us that\\
$\lim_{j\to +\infty} \mathbb{P}(\mu(\theta_a(\tau_j))<\tau_a) = 0$,
we deduce that there exists a constant $C\in ]0, 1[$ such that for all $j\ge 0$, $\mathbb{P}(\mu(\theta_a(\tau_j))<\tau_a) \le C$.
Hence,
\begin{equation*}
  \sum_{t=1}^{+\infty} \mathbb{P}(N_a(t)\le t^b) \le \sum_{t=1}^{+\infty} (t^b + 1) C^{t^{1-b}-1} < + \infty.
\end{equation*}


\subsection{Proof of \ref{lem:N_larger_tb}}

  In all this proof we consider a fixed profitable arm $a\in A^*$. We follow the lines of the proof of Proposition 5 in \cite{korda2013thompson} : let $t_j$ be the occurence of the $j$-th play of the arm $a$ (with $t_0=0$ by convention).
  Let $\xi_j = t_{j+1} - t_j - 1$, it corresponds to the number of time steps between the $j$-th and the $(j+1)$-th play of arm $a$. Hence, $t - N_a(t) \le \sum_{j=0}^{N_a(t)} \xi_j$ and we have
  \begin{equation*}
    \begin{split}
      \mathbb{P}(N_a(t)\le t^b) &\le \mathbb{P}(\exists j\in \{ 0, \dots, \lfloor t^b \rfloor, \xi_j \ge t^{1-b}-1 \})\\
      &\le \sum_{j=0}^{\lfloor t^b \rfloor} \mathbb{P}(\xi_j \ge t^{1-b}-1)\\
      &\le \sum_{j=0}^{\lfloor t^b \rfloor} \mathbb{P}(\mu(\theta_a(\tau_j)) < \tau_a)^{t^{1-b}-1}.
    \end{split}
  \end{equation*}

\end{document}